\documentclass[nohyperref]{article}
\usepackage[letterpaper, left=1in, right=1in, top=1in, bottom=1in]{geometry}

\usepackage{natbib}
\bibpunct{(}{)}{;}{a}{,}{,}
\usepackage[colorlinks,citecolor=blue!70!black,linkcolor=blue!70!black,
urlcolor=blue!70!black,breaklinks=true]{hyperref}

\usepackage{parskip}
\usepackage{breakcites}
\usepackage[noend]{algorithmic}

\usepackage{microtype}
\usepackage{graphicx}
\usepackage{subfigure}
\usepackage{booktabs}

\usepackage{amsmath}
\usepackage{amssymb}
\usepackage{mathtools}
\usepackage{amsthm}

\usepackage{amsmath, mathtools}
\usepackage{amsfonts, dsfont} 
\usepackage{url}            
\usepackage{algorithm, algorithmic}
\usepackage{amsthm}
\usepackage{listings}
\usepackage{bm}
\usepackage{bbm}
\usepackage{enumitem}
\usepackage{nicefrac}       \usepackage{mathrsfs} 
\usepackage{inconsolata}
\usepackage{comment}
\usepackage{xspace}

\usepackage{graphicx}
\usepackage{wrapfig}
\usepackage{bigints}
\usepackage{transparent}

\usepackage[nameinlink, capitalize]{cleveref}
\usepackage{thm-restate}
\usepackage{nicematrix}
\usepackage{arydshln}
\usepackage{mathtools}

\newtheorem{theorem}{Theorem}
\newtheorem{lemma}{Lemma}

\newtheorem{definition}{Definition}
\newtheorem{assumption}{Assumption}
\newtheorem{proposition}{Proposition}

\newtheorem{remark}{Remark}

\crefformat{equation}{#2Eq. (#1)#3}
\Crefformat{equation}{#2Eq. (#1)#3}

\Crefformat{figure}{#2Figure #1#3}
\Crefformat{assumption}{#2Assumption #1#3}
 
\makeatletter
\let\save@mathaccent\mathaccent
\newcommand*\if@single[3]{\setbox0\hbox{${\mathaccent"0362{#1}}^H$}\setbox2\hbox{${\mathaccent"0362{\kern0pt#1}}^H$}\ifdim\ht0=\ht2 #3\else #2\fi
  }
\newcommand*\rel@kern[1]{\kern#1\dimexpr\macc@kerna}
\newcommand*\widebar[1]{\@ifnextchar^{{\wide@bar{#1}{0}}}{\wide@bar{#1}{1}}}
\newcommand*\wide@bar[2]{\if@single{#1}{\wide@bar@{#1}{#2}{1}}{\wide@bar@{#1}{#2}{2}}}
\newcommand*\wide@bar@[3]{\begingroup
  \def\mathaccent##1##2{\let\mathaccent\save@mathaccent
    \if#32 \let\macc@nucleus\first@char \fi
    \setbox\z@\hbox{$\macc@style{\macc@nucleus}_{}$}\setbox\tw@\hbox{$\macc@style{\macc@nucleus}{}_{}$}\dimen@\wd\tw@
    \advance\dimen@-\wd\z@
    \divide\dimen@ 3
    \@tempdima\wd\tw@
    \advance\@tempdima-\scriptspace
    \divide\@tempdima 10
    \advance\dimen@-\@tempdima
    \ifdim\dimen@>\z@ \dimen@0pt\fi
    \rel@kern{0.6}\kern-\dimen@
    \if#31
      \overline{\rel@kern{-0.6}\kern\dimen@\macc@nucleus\rel@kern{0.4}\kern\dimen@}\advance\dimen@0.4\dimexpr\macc@kerna
      \let\final@kern#2\ifdim\dimen@<\z@ \let\final@kern1\fi
      \if\final@kern1 \kern-\dimen@\fi
    \else
      \overline{\rel@kern{-0.6}\kern\dimen@#1}\fi
  }\macc@depth\@ne
  \let\math@bgroup\@empty \let\math@egroup\macc@set@skewchar
  \mathsurround\z@ \frozen@everymath{\mathgroup\macc@group\relax}\macc@set@skewchar\relax
  \let\mathaccentV\macc@nested@a
  \if#31
    \macc@nested@a\relax111{#1}\else
    \def\gobble@till@marker##1\endmarker{}\futurelet\first@char\gobble@till@marker#1\endmarker
    \ifcat\noexpand\first@char A\else
      \def\first@char{}\fi
    \macc@nested@a\relax111{\first@char}\fi
  \endgroup
}
\makeatother

\newcommand{\approxleq}{\lesssim}

\newcommand{\bigoh}{O}
\newcommand{\bigoht}{\wt{O}}
\newcommand{\bigom}{\Omega}
\newcommand{\bigomt}{\wt{\Omega}}

\newcommand{\pref}[1]{\cref{#1}}
\newcommand{\pfref}[1]{Proof of \pref{#1}}
\newcommand{\polylog}{\mathrm{polylog}}

\newcommand{\indic}{\mathbb{I}}

\newcommand{\algcommentlight}[1]{\textcolor{blue!70!black}{\transparent{0.5}\small{\texttt{\textbf{//\hspace{2pt}#1}}}}}

\DeclarePairedDelimiter{\abs}{\lvert}{\rvert} \DeclarePairedDelimiter{\brk}{[}{]}
\DeclarePairedDelimiter{\crl}{\{}{\}}
\DeclarePairedDelimiter{\prn}{(}{)}

\DeclarePairedDelimiter{\ceil}{\lceil}{\rceil}

\DeclarePairedDelimiterX{\infdiv}[2]{(}{)}{#1\;\delimsize\|\;#2}

\let\P\undefined

\DeclareMathOperator{\P}{P}

\newcommand{\wt}[1]{\widetilde{#1}}
\newcommand{\wh}[1]{\widehat{#1}}
\newcommand{\wb}[1]{\widebar{#1}}

\def\ddefloop#1{\ifx\ddefloop#1\else\ddef{#1}\expandafter\ddefloop\fi}
\def\ddef#1{\expandafter\def\csname bb#1\endcsname{\ensuremath{\mathbb{#1}}}}
\ddefloop ABCDEFGHIJKLMNOPQRSTUVWXYZ\ddefloop
\def\ddefloop#1{\ifx\ddefloop#1\else\ddef{#1}\expandafter\ddefloop\fi}
\def\ddef#1{\expandafter\def\csname b#1\endcsname{\ensuremath{\mathbf{#1}}}}
\ddefloop ABCDEFGHIJKLMNOPQRSTUVWXYZ\ddefloop
\def\ddef#1{\expandafter\def\csname sf#1\endcsname{\ensuremath{\mathsf{#1}}}}
\ddefloop ABCDEFGHIJKLMNOPQRSTUVWXYZ\ddefloop
\def\ddef#1{\expandafter\def\csname c#1\endcsname{\ensuremath{\mathcal{#1}}}}
\ddefloop ABCDEFGHIJKLMNOPQRSTUVWXYZ\ddefloop
\def\ddef#1{\expandafter\def\csname h#1\endcsname{\ensuremath{\widehat{#1}}}}
\ddefloop ABCDEFGHIJKLMNOPQRSTUVWXYZ\ddefloop
\def\ddef#1{\expandafter\def\csname hc#1\endcsname{\ensuremath{\widehat{\mathcal{#1}}}}}
\ddefloop ABCDEFGHIJKLMNOPQRSTUVWXYZ\ddefloop
\def\ddef#1{\expandafter\def\csname t#1\endcsname{\ensuremath{\widetilde{#1}}}}
\ddefloop ABCDEFGHIJKLMNOPQRSTUVWXYZ\ddefloop
\def\ddef#1{\expandafter\def\csname tc#1\endcsname{\ensuremath{\widetilde{\mathcal{#1}}}}}
\ddefloop ABCDEFGHIJKLMNOPQRSTUVWXYZ\ddefloop
\def\ddefloop#1{\ifx\ddefloop#1\else\ddef{#1}\expandafter\ddefloop\fi}
\def\ddef#1{\expandafter\def\csname scr#1\endcsname{\ensuremath{\mathscr{#1}}}}
\ddefloop ABCDEFGHIJKLMNOPQRSTUVWXYZ\ddefloop

\newcommand{\ind}{\mathbbm{1}}

\newcommand{\ldef}{\vcentcolon=}

\DeclareMathOperator*{\argmin}{arg\,min}

\DeclareMathOperator{\unif}{{unif}}

\DeclareMathOperator{\dom}{{dom}}

\def\E{{\mathbb E}}

\def\P{{\mathbb P}}
\def\R{{\mathbb R}}
\def\N{{\mathbb N}}

\let\oldparagraph\paragraph
\renewcommand{\paragraph}[1]{\oldparagraph{#1.}}

\renewcommand{\epsilon}{\varepsilon}

\newcommand{\RegCB}{\mathrm{\mathbf{Reg}}_{\mathsf{CB}}}

\newcommand{\RegSq}{\mathrm{\mathbf{Reg}}_{\mathsf{Sq}}}
\newcommand{\AlgSq}{\mathrm{\mathbf{Alg}}_{\mathsf{Sq}}}
\newcommand{\AlgSample}{\mathrm{\mathbf{Alg}}_{\mathsf{Sample}}}

\newcommand{\dec}{\mathsf{dec}}
\newcommand{\SqAlg}{\AlgSq}

\newcommand{\regcb}{\RegCB}
\newcommand{\regcbh}{\mathrm{\mathbf{Reg}}_{\mathsf{CB},h}}
\newcommand{\regcbhb}{\mathrm{\mathbf{Reg}}_{\mathsf{CB},h_b}}
\newcommand{\regimphb}{\mathrm{\mathbf{Reg}}_{\mathsf{Imp},h_b}}
\newcommand{\regimph}{\mathrm{\mathbf{Reg}}_{\mathsf{Imp},h}}
\newcommand{\regsq}{\RegSq}
\newcommand{\sqalgtext}{$\SqAlg$\xspace}
\newcommand{\sqalg}{\sqalgtext}
\newcommand{\samplealg}{\AlgSample}
\newcommand{\samplealgtext}{$\samplealg$\xspace}

\newcommand{\smthigw}{\textsf{SmoothIGW}\xspace}

\newcommand{\Tsq}{\cT_{\mathsf{Sq}}}
\newcommand{\Tsample}{\cT_{\mathsf{Sample}}}
\newcommand{\Msq}{\cM_{\mathsf{Sq}}}
\newcommand{\Msample}{\cM_{\mathsf{Sample}}}

\newcommand{\curly}{\crl}
\newcommand{\paren}{\prn}

\newcommand{\sq}{\brk}

\newcommand{\dectext}{Decision-Estimation Coefficient\xspace}
\newcommand{\sgm}{\Omega}
\newcommand{\smthh}{\mathsf{Smooth}_h\xspace}
\newcommand{\smthhb}{\mathsf{Smooth}_{h_b}\xspace}
\newcommand{\smth}{\text{smooth}\xspace}

\newcommand{\cats}{\textsf{CATS}\xspace}

\newcommand{\greedy}{$\epsilon$-\textsf{Greedy}\xspace}
\newcommand{\corral}{\textsf{CORRAL}\xspace}

\title{Contextual Bandits with Smooth Regret: Efficient Learning in Continuous Action Spaces}
\date{}
\author{
Yinglun Zhu\\
{\normalsize University of Wisconsin-Madison}\\
{\small\texttt{yinglun@cs.wisc.edu}}
\and
\and
Paul Mineiro\\
{\normalsize Microsoft Research, NYC}\\
{\small\texttt{pmineiro@microsoft.com}}
}

\begin{document}

\maketitle
\begin{abstract}
	Designing efficient general-purpose contextual bandit algorithms that work with large---or even continuous---action spaces would facilitate application to important scenarios such as information retrieval, recommendation systems, and continuous control. While obtaining standard regret guarantees can be hopeless, alternative regret notions have been proposed to tackle the large action setting. We propose a smooth regret notion for contextual bandits, which dominates previously proposed alternatives. We design a statistically and computationally efficient algorithm---for the proposed smooth regret---that works with general function approximation under standard supervised oracles. We also present an adaptive algorithm that automatically adapts to any smoothness level. Our algorithms can be used to recover the previous minimax/Pareto optimal guarantees under the standard regret, e.g., in bandit problems with multiple best arms and Lipschitz/H{\"o}lder bandits. We conduct large-scale empirical evaluations demonstrating the efficacy of our proposed algorithms.
\end{abstract}

\section{Introduction}

Contextual bandits concern the problem of sequential decision making with contextual information.
Provably efficient contextual bandit algorithms have been proposed over the past decade \citep{langford2007epoch,agarwal2014taming,foster2020beyond,simchi2021bypassing,foster2021efficient}.
However, these developments only work in setting with a small number of actions,
and their theoretical guarantees become vacuous when working with a large action space \citep{agarwal2012contextual}. 
The hardness result can be intuitively understood through a ``needle in the haystack'' construction: 
When good actions are extremely rare, identifying any good action demands trying almost all alternatives.
This prevents naive direct application of contextual bandit algorithms to large action problems, e.g., in information retrieval, recommendation systems, and continuous control.

To bypass the hardness result, one approach is to assume structure on the model class.
For example, in the standard linear contextual bandit \citep{auer2002using, chu2011contextual, abbasi2011improved}, learning the $d$ components of the reward vector---rather than examining every single action---effectively guides the learner to the optimal action.
Additional structural assumptions have been studied in the literature, e.g., linearly structured actions and general function approximation \citep{foster2020adapting, xu2020upper}, Lipschitz/H\"older regression functions \citep{kleinberg2004nearly, hadiji2019polynomial}, and convex functions \citep{lattimore2020improved}.
While these assumptions are fruitful theoretically, they might be violated in practice.

An alternative approach is to compete against a less demanding benchmark.
Rather than competing against a policy that always plays the best action, one can compete against a policy that plays the best smoothed distribution over the actions: a smoothed distribution---by definition---cannot concentrate on the best actions when they are in fact rare.  Thus, for the previously mentioned ``needle in the haystack'' construction, the benchmark is weak as well.  This de-emphasizes such constructions and focuses algorithm design on scenarios where intuition suggests good solutions can be found without prohibitive statistical cost.

\paragraph{Contributions}
We study large action space problems under an alternate notion of regret.
Our first contribution is to propose a novel benchmark---the \smth regret---that formalizes the ``no needle in the haystack'' principle.  
We also show that our \smth regret dominates previously proposed regret notions along this line of work \citep{chaudhuri2018quantile, krishnamurthy2020contextual, majzoubi2020efficient}, i.e., any regret guarantees with respect to the \smth regret automatically holds for these previously proposed regrets.

We design efficient algorithms that work with the \smth regret and general function classes. 
Our first proposed algorithm, \smthigw, works with any fixed smoothness level $h>0$, and is efficient---both statistically and computationally---whenever the learner has access to standard oracles: (i) an online regression oracle for supervised learning, and (ii) a simple sampling oracle over the action space.
Statistically, \smthigw achieves $\sqrt{T /h}$-type regret for whatever action spaces; here ${1}/{h}$ should be viewed as the effective number of actions. Such guarantees can be verified to be minimax optimal when related back to the standard regret.
Computationally, the guarantee is achieved with $O(1)$ operations with respect to oracles, which can be usually efficiently implemented in practice.
Our second algorithm is a master algorithm which combines multiple \smthigw instances to compete against any unknown smoothness level. We show this master algorithm is Pareto optimal.

With our \smth regret and proposed algorithms, we exhibit guarantees under the standard regret in various scenarios, e.g., in problems with multiple best actions \citep{zhu2020regret} and in problems when the expected payoff function satisfies structural assumptions such as Lipchitz/H\"{o}lder continuity \citep{kleinberg2004nearly, hadiji2019polynomial}. Our algorithms are minimax/Pareto optimal when specialized to these settings.

\subsection{Paper Organization}
We introduce our smooth regret in \cref{sec:setting}, together with statistical and computational oracles upon which our algorithms are built. 
In \cref{sec:alg}, we present our algorithm \smthigw, which illustrates the core ideas of learning with smooth regret at any fixed smoothness level. 
Built upon \smthigw, in \cref{sec:adaptive}, we present a \corral-type of algorithm that can automatically adapt to any unknown smoothness level.
In \cref{sec:extension}, we connect our proposed smooth regret to the standard regret over various scenarios.
We present empirical results in \cref{sec:experiments}, and close with a discussion in \cref{sec:discussion}.

\section{Problem Setting}
\label{sec:setting}

We consider the following standard contextual bandit problems. At any time step $t \in [T]$, nature selects a context $x_t \in \cX$ and a distribution over loss functions $\ell_t: \cA \rightarrow [0,1]$ mapping from the (compact) action set $\cA$ to a loss value in $[0, 1]$.  Conditioned on the context $x_t$, the loss function is stochastically generated, i.e., $\ell_t \sim \P_{\ell_t}(\cdot \mid x_t)$.
The learner selects an action $a_t \in \cA$ based on the revealed context $x_t$, and obtains (only) the loss $\ell_t(a_t)$ of the selected action. 
The learner has access to a set of measurable regression functions $\cF \subseteq (\cX \times \cA \rightarrow [0,1])$ to predict the loss of any context-action pair.
We make the following standard realizability assumption studied in the contextual bandit literature \citep{agarwal2012contextual, foster2018practical, foster2020beyond, simchi2021bypassing}.
\begin{assumption}[Realizability]
\label{asm:realizability}
There exists a regression function $f^\star \in \cF$ such that $ \E \sq{\ell_t(a) \mid x_t} = f^\star(x_t, a)$ for any $a \in \cA$ and across all $t \in [T]$.
\end{assumption}

\paragraph{The smooth regret}
Let $(\cA, \sgm)$ be a measurable space of the action set and $\mu$ be a base probability measure over the actions. 
Let $\cQ_h$ denote the set of probability measures such that, for any measure $Q \in \cQ_h$, the following holds true: (i) $Q$ is absolutely continuous with respect to the base measure  $\mu$, i.e., $Q \ll \mu$; and (ii) The Radon-Nikodym derivative of $Q$ with respect to $\mu$ is no larger than $\frac{1}{h}$, i.e., $\frac{dQ}{d\mu} \leq 1/h$.
We call $\cQ_h$ the set of {smoothing kernels at smoothness level $h$, or simply put the set of $h$-smoothed kernels.}
For any context $x \in \cX$, we denote by $\smthh(x)$ the smallest loss incurred by any  $h$-smoothed kernel, i.e.,  
\begin{align*}
    \smthh(x) \ldef \inf_{Q \in \cQ_h}\E_{a \sim Q}\sq{f^\star(x,a)}.
\end{align*}
Rather than competing with $\argmin_{a\in\cA}f^\star(x,a)$---an impossible job in many cases---we take $\smthh(x)$ as the benchmark and define the \emph{smooth regret} as follows:
\begin{align}
    \regcbh(T) & \coloneqq  
    \E \sq*{ \sum_{t=1}^T  f^\star(x_t, a_t) - \smthh(x_t) } \label{eq:smooth_regret}.
\end{align}
One important feature about the above definition is that the benchmark, i.e., $\smthh(x_t)$, automatically adapts to the context $x_t$: 
This gives the benchmark more power and makes it harder to compete against. 
In fact, our \smth regret dominates many existing regret measures with \emph{easier} benchmarks. We provide some examples in the following. 
\begin{itemize}
    \item \citet{chaudhuri2018quantile} propose the quantile regret, which aims at competing with the lower $h$-quantile of the loss function, i.e., $v_h(x) \ldef \inf \crl{\zeta: \mu(a \in \cA: f^\star(x,a) \leq \zeta) \geq h}$.
	    Consider $\cS_h \ldef \crl{a \in \cA: f^{\star}(x,a) \leq \nu_h(x)}$ such that $\mu(\cS_h) \geq h$.
	   Let $\wb Q_h \ldef \mu \vert_{\cS_h} / \mu(\cS_h)$ denote the (normalized) probability measure after restricting $\mu$ onto  $\cS_h$.
Since $\wb Q_h \in \cQ_h$, we clearly have $\smthh(x) \leq \E_{a \sim \wb Q_h} \brk{f^{\star}(x,a)} \leq \nu_h(x)$.
Besides, the (original) quantile was only studied in the non-contextual case.
    \item 
\citet{krishnamurthy2020contextual} study a notion of regret that is smoothed in a different way:
Their regret aims at competing with a known and \emph{fixed} smoothing kernel (on top of a {fixed} policy set) with Radon-Nikodym derivative at most ${1}/{h}$. Our benchmark is clearly harder to compete against since we consider any smoothing kernel with Radon-Nikodym derivative at most ${1}/{h}$. 
\end{itemize}
Besides being more competitive with respect to above benchmarks, \smth regret can also be naturally linked to the \emph{standard} regret under various settings previously studied in the bandit literature, e.g., 
in the discrete case with multiple best arms \citep{zhu2020regret} and in the continuous case with Lipschitz/H\"{o}lder continuous payoff functions \citep{kleinberg2004nearly,hadiji2019polynomial}.
We provide detailed discussion in \cref{sec:extension}.

\subsection{Computational Oracles}
The first step towards designing computationally efficient algorithms is to identify reasonable oracle models to access the sets of regression functions or actions. Otherwise, enumeration over regression functions or actions (both can be exponentially large) immediately invalidate the computational efficiency. 
We consider two common oracle models: a regression oracle and a sampling oracle.

\paragraph{The regression oracles}
A fruitful approach to designing efficient contextual bandit algorithms is through reduction to supervised regression with the class $\cF$ \citep{foster2020beyond, simchi2021bypassing, foster2020adapting, foster2021instance}. 
Following \citet{foster2020beyond}, we assume that we have access to an \emph{online} regression oracle \sqalgtext, which is an algorithm for sequential predication under square loss.
More specifically, the oracle operates in the following protocol: At each round $t \in [T]$, the oracle makes a prediction $\wh f_t$, then receives context-action-loss tuple $(x_t, a_t, \ell_t(a_t))$.
The goal of the oracle is to accurately predict the loss as a function of the context and action, and we evaluate its performance via the square loss $\prn{\wh f_t(x_t,a_t) - \ell_t(a_t)}^2$.
We measure the oracle's cumulative performance through the square-loss regret to $\cF$, which is formalized below.

\begin{assumption}
\label{assumption:regression_oracle}
The regression oracle \sqalgtext guarantees that, with probability at least $1-\delta$, for any (potentially adaptively chosen) sequence $\curly*{(x_t, a_t, \ell_t(a_t))}_{t=1}^T$, 
\begin{align*}
	\E \Bigg[\sum_{t=1}^T &  \prn*{\wh f_t(x_t, a_t) - \ell_t(a_t)}^2 - 
      \inf_{f \in \cF} \sum_{t=1}^T \prn*{f(x_t, a_t) - \ell_t(a_t)}^2 \Bigg]
    \leq \regsq(T, \delta),
\end{align*}
for some (non-data-dependent) function $\regsq(T, \delta)$.
\end{assumption}

Sometimes it's useful to consider a \emph{weighted} regression oracle, where the square errors are weighted differently. It is shown in \citet{foster2020adapting} (Theorem 5 therein) that any regression oracle satisfies \cref{assumption:regression_oracle} can be used to generate a weighted regression oracle that satisfies the following assumption.

\begin{assumption}
\label{assumption:regression_oracle_weighted}
The regression oracle \sqalg guarantees that, with probability at least $1-\delta$, for any (potentially adaptively chosen) sequence $\curly*{(w_t, x_t, a_t, \ell_t(a_t))}_{t=1}^T$, 
\begin{align*}
	\E \Bigg[ \sum_{t=1}^T & w_t  \prn*{\wh f_t(x_t, a_t) - \ell_t(a_t)}^2 -
     \inf_{f \in \cF} \sum_{t=1}^T w_t \prn*{f(x_t, a_t) - \ell_t(a_t)}^2 \Bigg]
     \leq \E \sq*{ \max_{t \in [T]} w_t }  \regsq(T, {\delta}),
\end{align*}
for some (non-data-dependent) function $\regsq(T, \delta)$.
\end{assumption}

{For either regression oracle,} we let $\Tsq$ denote an upper bound on the time to (i) query the oracle's estimator $\wh f_t$ with context-action pair $(x_t,a)$ and receive its predicated value $\wh f_t(x_t,a) \in [0,1]$; (ii) query the oracle's estimator $\wh f_t$ with context $x_t$ and receive its argmin action $\wh a_t = \argmin_{a \in \cA} \wh f_t(x_t,a)$; and (iii) update the oracle with example $(x_t, a_t, r_t(a_t))$.
We let $\Msq$ denote the maximum memory used by the oracle throughout its execution.

Online regression is a well-studied problem, with known algorithms for many model classes \citep{foster2020beyond, foster2020adapting}: including linear models \citep{hazan2007logarithmic}, generalized linear models \citep{kakade2011efficient}, non-parametric models \citep{gaillard2015chaining}, and beyond. 
Using Vovk's aggregation algorithm \citep{vovk1998game}, one can show that $\regsq(T, \delta) = O(\log \prn{\abs{\cF} /\delta})$ for any finite set of regression functions $\cF$, which is the canonical setting studied in contextual bandits \citep{langford2007epoch, agarwal2012contextual}.
In the following of this paper, we use abbreviation $\regsq(T) \ldef \regsq(T, T^{-1})$, and will keep the $\regsq(T)$ term in our regret bounds to accommodate for general set of regression functions.

\paragraph{The sampling oracles}
In order to design algorithms that work with large/continuous action spaces, we assume access to a
sampling oracle \samplealgtext to get access to the action space. 
In particular, the oracle \samplealgtext returns an action $a \sim \mu$ randomly drawn according to the base probability measure $\mu$ over the action space $\cA$.  
We let $\Tsample$ denote a bound on the runtime of single query to the oracle; and let $\Msample$ denote the maximum memory used by the oracle. 

\paragraph{Representing the actions} We use $b_\cA$ to denote the number of bits required to represent any action $a\in \cA$, which scales with $O(\log\abs{\cA})$ with a finite set of actions and $\wt O(d)$ for actions represented as vectors in $\R^d$.
Tighter bounds are possible with additional structual assumptions. Since representing actions is a minimal assumption, we hide the dependence on $b_\cA$ in big-$\bigoh$ notation for our runtime and memory analysis.

\subsection{Additional Notation}
    We adopt non-asymptotic big-oh notation: For functions
	$f,g:\cZ\to \R_{+}$, we write $f=\bigoh(g)$ (resp. $f=\bigom(g)$) if there exists a constant
	$C>0$ such that $f(z)\leq{}Cg(z)$ (resp. $f(z)\geq{}Cg(z)$)
        for all $z\in\cZ$. We write $f=\bigoht(g)$ if
        $f=\bigoh(g\cdot\mathrm{polylog}(T))$, $f=\bigomt(g)$ if $f=\bigom(g/\polylog(T))$.
         We use $\lesssim$ only in informal statements to
highlight salient elements of an inequality.

  For an integer $n\in\bbN$, we let $[n]$ denote the set
        $\{1,\dots,n\}$.  
        For a set $\cZ$, we let
        $\Delta(\cZ)$ denote the set of all Radon probability measures
        over $\cZ$. 
        We let $\unif(\cZ)$ denote the uniform distribution/measure over 
        $\cZ$.
	We let $\indic_{z}\in\Delta(\cZ)$ denote the delta distribution on $z$.

\section{Efficient Algorithm with Smooth Regret}
\label{sec:alg}

We design an oracle-efficient (\smthigw, \cref{alg:smooth}) algorithm that achieves a $\sqrt{T}$-type regret under the \smth regret defined in \cref{eq:smooth_regret}. 
We focus on the case when the smoothness level $h>0$ is known in this section,
and leave the design of adaptive algorithms in \cref{sec:adaptive}.

\cref{alg:smooth} contains the pseudo code of our proposed \smthigw algorithm, which deploys a smoothed sampling distribution to balance exploration and exploitation.
At each round $t \in [T]$, the learner observes the context $x_t$ from the environment and obtains the estimator $\widehat f_t$ from the regression oracle \sqalgtext.
It then constructs a sampling distribution $P_t$ by mixing a smoothed distribution constructed using the \emph{inverse gap weighting} (IGW) technique \citep{abe1999associative, foster2020beyond} and a delta mass at the greedy action $\wh a_t \ldef \argmin_{a \in \cA} \wh f_t(x_t, a)$.
The algorithm samples an action $a_t \sim P_t$ and then update the regression oracle \sqalgtext. 
The key innovation of the algorithm lies in the construction of the smoothed IGW distribution, which we explain in detail next.

\begin{algorithm}[]
	\caption{\smthigw}
	\label{alg:smooth} 
	\renewcommand{\algorithmicrequire}{\textbf{Input:}}
	\renewcommand{\algorithmicensure}{\textbf{Output:}}
	\newcommand{\algorithmicbreak}{\textbf{break}}
    \newcommand{\BREAK}{\STATE \algorithmicbreak}
	\begin{algorithmic}[1]
		\REQUIRE Exploration parameter $\gamma > 0$, online regression oracle \sqalg.
		\FOR{$t = 1, 2, \dots, T$}
		\STATE Observe context $x_t$.
		\STATE Receive $\widehat f_t$ from regression oracle \sqalg.
		\STATE Get $\widehat a_t \gets \argmin_{a \in \cA} \widehat f_t(x_t, a)$.
		\STATE Define 
		\begin{align}
		P_t \ldef M_t +  (1-M_t(\cA)) \cdot \indic_{\wh a_t}, 
    \label{eq:sampling_dist}
		\end{align}
		where {$M_t$ is the measure defined in \cref{eq:abe_long_measure}}
		\STATE Sample $a_t \sim P_t$ and observe loss $\ell_t(a_t)$. \algcommentlight{This can be done efficiently via \cref{alg:reject_sampling}.}
        \STATE  Update \sqalg with $(x_t, a_t, \ell_t(a_t))$ 
		\ENDFOR
	\end{algorithmic}
\end{algorithm}

\paragraph{Smoothed variant of IGW} 
The IGW technique was previously used in the finite-action contextual bandit setting \citep{abe1999associative, foster2020beyond}, 
which assigns a probability mass to every action $a \in \cA$ inversely proportional to the estimated loss gap $(\wh f(x,a) - \wh f(x, \wh a))$. 
To extend this strategy to continuous action spaces we leverage Radon-Nikodym derivatives.  
Fix any constant $\gamma > 0$, we define a IGW-type function as 
\begin{align}
    m_t(a) \ldef \frac{1}{1+ h \gamma(\widehat f_t(x_t, a) - \widehat f_t(x_t, \widehat a_t))}. \label{eq:density_m}
\end{align} 
For any $\omega \in \sgm$, we then define a new measure 
\begin{align}
    M_t(\omega) \ldef \int_{a \in \omega} m_t(a) \, d \mu (a) \label{eq:abe_long_measure}
\end{align}
of the measurable action space $(\cA, \sgm)$, where $m(a)  = \frac{dM}{d\mu}(a)$ serves as the Radon-Nikodym derivative between the new measure $M$ and the base measure $\mu$.  
Since $m_t(a) \leq 1$ by construction, we have $M_t(\cA) \leq 1$, i.e., $M_t$ is a sub-probability measure.
\smthigw plays a probability measure $P_t \in \Delta(\cA)$ by mixing the sub-probability measure $M_t$ 
with a delta mass at the greedy action $\wh a_t$, as in \cref{eq:sampling_dist}.

\begin{algorithm}[]
	\caption{Rejection Sampling for IGW}
	\label{alg:reject_sampling} 
	\renewcommand{\algorithmicrequire}{\textbf{Input:}}
	\renewcommand{\algorithmicensure}{\textbf{Output:}}
	\newcommand{\algorithmicbreak}{\textbf{break}}
    \newcommand{\BREAK}{\STATE \algorithmicbreak}
	\begin{algorithmic}[1]
		\REQUIRE Sampling oracle \samplealgtext, greedy action $\wh a_t$, Radon-Nikodym derivative $m_t(a)$.
		\STATE Draw $a \sim \mu$ from sampling oracle \samplealgtext.
		\STATE Sample $Z$ from a Bernoulli random distribution with mean $m_t(a)$. 
		\IF{$Z=1$}
		\STATE Take action $a$.
	    \ELSE
		\STATE Take action $\wh a_t$.
		\ENDIF
	\end{algorithmic}
\end{algorithm}

\paragraph{Efficient sampling} We now discuss how to sample from the distribution of \cref{eq:sampling_dist} using a single call to the sampling oracle, via rejection sampling. We first randomly sample an action $a \sim \mu$ from the sampling oracle \samplealgtext and with respect to the base measure $\mu$. We then compute $m_t(a)$ in \cref{eq:density_m} with two evaluation calls to $\wh f_t$, one at $\wh f_t(x_t, {a})$ and the other at $\wh f_t(x_t, \wh a_t)$. Finally, we sample a random variable $Z$ from a Bernoulli distribution with expectation $m_t(a)$ and play either action $\wh{a}_t$ or action $a$ depending upon the realization of $Z$.
One can show that the sampling distribution described above coincides with the distribution defined in \cref{eq:sampling_dist} (\cref{prop:reject_sampling}).\footnote{The same idea can be immediately applied to the case of sampling from the IGW distribution with finite number of actions \citep{foster2020beyond}.}
We present the pseudo code for rejection sampling in \cref{alg:reject_sampling}.

\begin{proposition}
	\label{prop:reject_sampling}
    The sampling distribution generated from \cref{alg:reject_sampling}
    coincides with the sampling distribution defined in \cref{eq:sampling_dist}. 
\end{proposition}
\begin{proof}[\pfref{prop:reject_sampling}]
    Let $\wb P_t$ denote the sampling distribution achieved by \cref{alg:reject_sampling}.
    For any $\omega \in \sgm $, if $\wh a_t \notin \omega$, we have 
    \begin{align*}
        \wb P_t(\omega) = \int_{a \in \omega} m_t(a)\, d\mu(a) = M_t(\omega)
    \end{align*} 
    Now suppose that $\wh a_t \in \omega$: Then the rejection probability, which equals $\E_{a \sim \mu} \sq*{ 1 - m_t(a)} = 1 - M_t(\cA)$, will be added to the above expression. \end{proof}

We now state the regret bound for \smthigw in the following.
\begin{restatable}{theorem}{thmRegret}
\label{thm:regret}
Fix any smoothness level $h \in (0,1]$.
With an appropriate choice for  $\gamma > 0$, \cref{alg:smooth} ensures that 
\begin{align*}
    \regcbh(T) \leq {\sqrt{4 T \, \regsq(T)/h} },
\end{align*}
with per-round runtime $O(\Tsq + \Tsample)$ and maximum memory $O(\Msq + \Msample)$.
\end{restatable}

\paragraph{Key features of \cref{alg:smooth}}
\cref{alg:smooth} achieves $\wt O \prn{\sqrt{T /h}}$ regret, which has no dependence on the number of actions.\footnote{We focus on the canonical case studied in contextual bandits with a finite $\cF$, and view $\regsq(T) = O(\log \abs{\cF})$.}  This suggests the \cref{alg:smooth} can be used in large action spaces scenarios and only suffer regret scales with $1 /h$: the effective number of actions considered for \smth regret.
We next highlight the statistical and computational efficiencies of \cref{alg:smooth}.
\begin{itemize}
    \item \emph{Statistical optimality.} 
	    It's not hard to prove a $\wt \Omega(\sqrt{T /h})$ lower bound for the \smth regret by relating it to the standard regret under a contextual bandit problem with finite actions: (i) the \smth regret and the standard regret coincides when $h = {1}/{\abs{\cA}}$; and (ii) the standard regret admits lower bound $\wt \Omega(\sqrt{\abs{\cA}T})$ \citep{agarwal2012contextual}.
	    In \cref{sec:extension}, we further relate our \smth regret guarantee to standard regret guarantee under other scenarios and recover the minimax bounds.

    \item \emph{Computational efficiency.} \cref{alg:smooth} is oracle-efficient and enjoys per-round runtime and maximum memory that scales linearly with oracle costs. To our knowledge, this leads to the first computationally efficient general-purpose algorithm that achieves a $\sqrt{T}$-type guarantee under \smth regret. The previously known efficient algorithm applies an \greedy-type of strategy and thus only achieves a $T^{2/3}$-type regret (\citet{majzoubi2020efficient}, and with respect to a weaker version of the \smth regret).
	
\end{itemize}

\paragraph{Proof sketch for \cref{thm:regret}}
To analyze \cref{alg:smooth}, we follow a recipe introduced by \citet{foster2020beyond, foster2020adapting, foster2021statistical} based on the \emph{\dectext} (DEC, adjusted to our setting), defined as $\dec_\gamma(\cF) \ldef \sup_{\wh{f} , x } \dec_{\gamma}(\cF;\wh{f},x)$, where
\begin{align}
	 \dec_{\gamma}(\cF; \wh{f}, x) \ldef \inf_{P \in \Delta(\cA)} \sup_{f^\star \in \cF} \E_{a \sim P}  
	 \bigg[ f^\star(x, a^\star) - \smthh(x) - \frac{\gamma}{4} \cdot \prn[\big]{\wh{f}(x,a) - f^\star(x,a)}^2 \bigg] .
\label{eq:dec}
\end{align}
\citet{foster2020beyond, foster2020adapting, foster2021statistical} consider a meta-algorithm which, at each round $t$, (i) computes $\wh f_t$ by appealing to a regression oracle, (ii) computes a distribution $P_t\in\Delta(\cA)$ that solves the minimax problem in \pref{eq:dec} with $x_t$ and $\wh f_t$ plugged in, and (iii) chooses the action $a_t$ by sampling from this distribution. One can show that for any $\gamma > 0$, this strategy enjoys the following regret bound:
\begin{align}
	\regcbh(T) \approxleq T \cdot \dec_\gamma(\cF) + \gamma \cdot \regsq(T), \label{eq:decomposition}
\end{align}
More generally, if one computes a distribution that does not solve \pref{eq:dec} exactly, but instead certifies an upper bound on the DEC of the form $\dec_\gamma(\cF) \leq \wb \dec_\gamma(\cF)$, the same result holds with $\dec_\gamma(\cF)$ replaced by $\wb \dec_\gamma(\cF)$. \cref{alg:smooth} is a special case of this meta-algorithm, so to bound the regret it suffices to show that the exploration strategy in the algorithm certifies a bound on the DEC.

By applying principles of convex conjugate, we show that the IGW-type distribution of \cref{eq:sampling_dist} certifies $\dec_\gamma(\cF) \leq \frac{2}{h\gamma}$ for any set of regression functions $\cF$ (\cref{prop:dec_bound}, deferred to \cref{app:smooth_supporting}).
With this bound on DEC, We can then bound the first term in \cref{eq:decomposition} by $O(\frac{T}{h \gamma})$ and optimally tune $\gamma$ in \cref{eq:decomposition} to obtain the desired regret guarantee.

Deriving the bound on the DEC is one of our key technical contributions, where we simultaneous eliminate the dependence on both the function class and (cardinality of) the action set.
Previous bounds on the DEC assume either a restricted function class $\cF$ or a finite action set.

\section{Adapting to Unknown Smoothness Parameters}
\label{sec:adaptive}

Our results in \cref{sec:alg} shows that, with a known $h$, one can achieve  \smth regret proportional to $\sqrt{T/h }$ against the optimal smoothing kernel in $\cQ_h$. The total loss achieved by the learner is the \smth regret plus the total loss suffered by playing the optimal smoothing kernel. One can notice that these two terms go into different directions: When $h$ gets smaller, the loss suffered by the optimal smoothing kernel gets smaller, yet the regret term gets larger. It is apriori unclear how to balance these terms, and therefore desirable to design algorithms that can automatically adapt to an unknown $h \in (0, 1]$. Note it is sufficient to adapt to unknown $h \in [1/T, 1]$, as the regret bound is vacuous for $h < 1/T$.  We provide such an algorithm in this section.

\paragraph{The \corral master algorithm}
Our algorithm follows the standard master-base algorithm structure: We run multiple base algorithms with different configurations in parallel, and then use a master algorithm to conduct model selection on top of base algorithms. The goal of the master algorithm is to balance the regret among base algorithms and eventually achieve a performance that is ``close'' to the best base algorithm (whose identity is unknown). We use the classical \corral algorithm \citep{agarwal2017corralling} as the master algorithm and initiate a collection of $B = \ceil*{\log T}$ (modified) \cref{alg:smooth} as base algorithms. More specifically, for $b=1,2,\dots,B$, each base algorithm is initialized with smoothness level $h_b = 2^{-b}$. For any $h^\star \in [1/T,1]$, one can notice that there exists a base algorithm $i^\star$ that suits well to this (unknown) $h^\star$ in the sense that $h_{b^\star} \leq h^\star \leq 2 h_{b^\star}$. The goal of the master algorithm is thus to adapt to the base algorithm indexed by $b^\star$.

We provide a brief description of the \corral master algorithm, and direct the reader to \citet{agarwal2017corralling} for more details. The master algorithm maintains a distribution $q_t \in \Delta([B])$ over base algorithms. At each round, the master algorithm sample a base algorithm $I_t \sim q_t$ and passes the context $x_t$, the sampling probability $q_{t,I_t}$ and parameter $\rho_{t,I_t} \coloneqq 1/ \min_{i \leq t} q_{t,I_t}$ into the base algorithm $I_t$. The base algorithm $I_t$ then performs its learning process: it samples an arm $a_t$, observes its loss $\ell_t(a_{t,I_t})$, and then updates its internal state. The master algorithm is updated with respect to the importance-weighted loss $\frac{\ell_t(a_{t,I_t})}{q_{t,I_t}}$ and parameter $\rho_{t,I_t}$. In order to obtain theoretical guarantees, the base algorithms are required to be stable, which is defined as follows.

\begin{definition}
\label{def:stable}
Suppose the base algorithm indexed by $b$ satisfies---when implemented alone---regret guarantee $\regcbhb(T) \leq R_{b}(T)$ for some non-decreasing $R_{b}(T):\N_+ \rightarrow \R_+$. 
Let $\regimph$ denote the \emph{importance-weighted} regret for base algorithm  $b$, i.e.,
\begin{align*}
\regimphb(T) \ldef 
\E \sq*{ \sum_{t = 1}^T \frac{\ind(I_t = b)}{q_{t,b}} \paren{ f^\star(x_t,a_t) - \smthhb(x_t) } }.
\end{align*}
The base algorithm $b$ is called $(\alpha, R_b(T))$ stable if 
    $\regimphb (T) \leq \E \sq*{\rho^\alpha_{T,b}} R_b(T)$. 
\end{definition}

\paragraph{A stable base algorithm}
Our treatment is inspired by \citet{foster2020adapting}. 
Let $\prn{\tau_1, \tau_2, \ldots} \subseteq [T]$ denote the time steps when the base algorithm $b$ is invoked,
i.e., when $I_t = b$. When invoked, the base algorithm receives $(x_t, q_{t,b}, \rho_{t,b})$ from the master algorithm. The base algorithm then sample from a distribution similar to \cref{eq:sampling_dist} but with a customized learning rate $\gamma_{t,b} \ldef \sqrt{8T/ (h_b \cdot \rho_{t,b} \cdot \regsq(T))}$. 
After observing the loss $\ell_t(a_{t, b})$, the base algorithm then updates the weighted regression oracle satisfying \cref{assumption:regression_oracle_weighted}. Our modified algorithm is summarized in \cref{alg:stable}.

\begin{algorithm}[H]
	\caption{Stable Base Algorithm (Index $b$)}
	\label{alg:stable} 
	\renewcommand{\algorithmicrequire}{\textbf{Input:}}
	\renewcommand{\algorithmicensure}{\textbf{Output:}}
	\newcommand{\algorithmicbreak}{\textbf{break}}
    \newcommand{\BREAK}{\STATE \algorithmicbreak}
	\begin{algorithmic}[1]
		\REQUIRE Weighted online regression oracle \sqalg.
		\STATE Initialize weighted regression oracle \sqalg.
		\FOR{$t \in \prn{\tau_1, \tau_2, \ldots} $}
		\STATE Receive context $x_t$, probability $q_{t,b}$ and parameter $\rho_{t,b}$ from the master algorithm.
		\STATE Receive $\widehat f_{t,b}$ from the \emph{weighted} online regression oracle \sqalg.
		\STATE Get $\widehat a_{t,b} \gets \argmin_{a \in \cA} \widehat f_{t,b}(x_t, a)$.
		\STATE Define $\gamma_{t,b} \ldef \sqrt{8 T/ (h_b \cdot \rho_{t,b} 
		\cdot \regsq(T))}$ and $w_{t,b} \ldef \ind(I_t=b) \cdot  \gamma_{t,b}/q_{t,b}$.
		\STATE Define $P_{t,b} \ldef M_{t,b} +  (1-M_{t,b}(\cA)) \cdot  \indic_{\wh a_{t,b}}$ according to \cref{eq:sampling_dist} but with $\gamma_{t,b}$ defined above.
        \STATE Sample $a_{t,b} \sim P_{t,b}$ and observe loss $\ell_t(a_{t,b})$. \algcommentlight{This can be done efficiently via \cref{alg:reject_sampling}.}
        \STATE  Update the weighted regression oracle \sqalg with $(w_{t,b}, x_t, a_t, \ell_t(a_{t,b}))$ 
		\ENDFOR
	\end{algorithmic}
\end{algorithm}

\begin{restatable}{proposition}{propStable}
\label{prop:stable}
For any $b \in [B]$, \cref{alg:stable} is $\prn*{\frac{1}{2},\sqrt{4T \, \regsq(T)/h_b}}$-stable, with per-round runtime $O(\Tsq + \Tsample)$ and maximum memory $O(\Msq + \Msample)$.
\end{restatable}

We now provide our model selection guarantees that adapt to unknown smoothness parameter $h \in (0,1]$. The result directly follows from combining the guarantee of \corral \citep{agarwal2017corralling} and our stable base algorithms.
\begin{restatable}{theorem}{thmAdaptive}
\label{thm:adaptive}
Fix learning rate $\eta \in (0,1]$, the \corral algorithm with \cref{alg:stable} as base algorithms guarantees that
\begin{align*}
     \regcbh(T) = \wt O \prn*{ \frac{1}{\eta} + \frac{\eta \, T  \, \regsq (T) }{h} }, \forall h \in (0,1].
\end{align*}
The \corral master algorithm has per-round runtime $\wt O(\Tsq + \Tsample)$ and maximum memory $\wt O(\Msq + \Msample)$.
\end{restatable}
\begin{remark}
\label{rm:adaptive_2}
We keep the current form of \cref{thm:adaptive} to better generalize to other settings, as explained in \cref{sec:extension}.
With a slightly different analysis, we can recover the $\wt O \prn{ T^{\frac{1}{1+\beta}} h^{-\beta} \prn*{\log \abs*{\cF }}^{\frac{\beta}{1+ \beta}}}$ guarantee for any $\beta \in [0,1]$, which is known to be Pareto optimal \citep{krishnamurthy2020contextual}. We provide the proofs for this result in \cref{app:adaptive_2}. 
\end{remark}

\section{Extensions to Standard Regret}
\label{sec:extension}
We extend our results to various settings under the standard regret guarantee, including the discrete case with multiple best arms, and the continuous case under Lipschitz/H\"{o}lder continuity. Our results not only recover previously known minimax/Pareto optimal guarantees, but also generalize existing results in various ways. 

Although our guarantees are stated in terms of the \smth regret, they are naturally linked to the standard regret among various settings studied in this section. We thus primarily focus on the standard regret in this section. Let $a^\star_t \ldef \argmin_{a \in \cA} f^\star(x_t, a)$ denote the best action under context $x_t$.
The \emph{standard} (expected) regret is defined as
\begin{align*}
    \regcb(T) & \ldef \E \sq*{\sum_{t=1}^T  f^{\star}(x_t, a_t) - f^{\star}(x_t, a^\star_t)}. \label{eq:regret}
\end{align*}
We focus on the canonical case with a finite set of regression functions $\cF$ and consider $\regsq(\cF) = O(\log \prn{\abs*{\cF}T})$ \citep{vovk1998game}.

\subsection{Discrete Case: Bandits with Multiple Best Arms}
\citet{zhu2020regret} study a non-contextual bandit problem with a large (discrete) action set $\cA$ which might contain multiple best arms. More specifically, suppose there exists a subset of optimal arms $\cA^\star \subseteq \cA$ with cardinalities $\abs*{\cA^\star}= K^\star$ and $\abs*{\cA} = K$, the goal is to adapt to the effective number of arms $\frac{K}{K^\star}$ and minimize the standard regret. Note that one could have $\frac{K}{K^\star} \ll K$ when $K^\star$ is large.

\textbf{Existing Results.} Suppose $\frac{K}{K^\star} = \Theta (T^\alpha)$ for some $\alpha \in [0,1]$. \citet{zhu2020regret} shows that: (i) when $\alpha$ is known, the minimax regret is $\wt \Theta(T^{(1+\alpha)/2})$; and (ii) when $\alpha$ is unknown, the Pareto optimal regret can be described by $\widetilde O( \max \curly*{T^{\beta}, T^{1+\alpha - \beta} })$ for any $\beta \in [0,1)$.

\textbf{Our Generalizations.} We extend the problem to the contextual setting: We use $\cA^\star_{x_t} \subseteq \cA$ to denote the \emph{subset} of optimal arms with respect to context $x_t$, and analogously assume that $\inf_{x \in \cX} \abs*{\cA^\star_{x}} = K^\star$ and $\frac{K}{K^\star}=T^\alpha$. 

Since $\frac{K^\star}{K}$ represents the proportion of actions that are optimal, by setting $h = \frac{K^\star}{K} = T^{-\alpha}$ (and under uniform measure), we can then relate the standard regret to the smooth regret, i.e., $\regcb(T) = \regcbh(T)$. In the case when $\alpha$ is known, \cref{thm:regret} implies that $\regcb(T) = O \prn[\big]{T^{(1+\alpha)/2} \log^{1/2} \prn{\abs*{\cF}T}}$. In the case with unknown $\alpha$, by setting $\eta = T^{-\beta}$ in \cref{thm:adaptive}, we have 
\begin{align*}
    \regcb(T) = O \paren[\big]{ \max \paren{ T^\beta, T^{1+\alpha-\beta} \log \prn{\abs*{\cF}T} }}.
\end{align*}
These results generalize the known minimax/Pareto optimal results in \citet{zhu2020regret} to the contextual bandit case, up to logarithmic factors.

\subsection{Continuous Case: Lipschitz/H\"older Bandits}

\citet{kleinberg2004nearly, hadiji2019polynomial} study non-contextual bandit problems with (non-contextual) mean payoff functions $f^\star(a)$ satisfying H\"older continuity. More specifically, let $\cA = [0,1]$ (with uniform measure) and $L, \alpha > 0$ be some H\"older smoothness parameters, the assumption is that
\begin{align*}
    \abs{ f^\star(a) - f^\star(a^\prime) } \leq L \, \abs{a - a^\prime}^\alpha,
\end{align*}
for any $a, a^\prime \in \cA$. The goal is to adapt to provide standard regret guarantee that adapts to the smoothness parameters $L$ and $\alpha$.

\textbf{Existing Results.} In the case when $L,\alpha$ are known, \citet{kleinberg2004nearly} shows that the minimax regret scales as $\Theta \prn{L^{1/(2\alpha+1)} T^{(\alpha+1)/(2\alpha+1)}}$; in the case with unknown $L, \alpha$, \citet{hadiji2019polynomial} shows that the Pareto optimal regret can be described by $\wt O \prn[\big]{\max \crl{T^\beta, L^{1/(1+\alpha)} T^{1- \frac{\alpha}{1+\alpha}\beta } }}$ for any $\beta \in [\frac{1}{2},1]$.

\textbf{Our Generalizations.} We extend the setting to the contextual bandit case and make the following analogous H\"older continuity assumption,\footnote{The special case with Lipschitz continuity ($\alpha = 1$) has been previously studied in the contextual setting, e.g., see \citet{krishnamurthy2020contextual}.} i.e.,
\begin{align*}
    \abs{ f^\star(x,a) - f^\star(x, a^\prime) } \leq L \, \abs{a - a^\prime}^\alpha, \quad \forall x \in \cX.
\end{align*}
We first divide the action set $\cA=[0,1]$ into $B=\ceil*{1/h}$ consecutive intervals $\crl{I_b}_{b=1}^{B}$ such that $I_b = [(b-1){h}, {b}{h}]$. Let $b_t$ denote the index of the interval where the best action $a^\star_t \ldef \argmin_{a \in \cA}f^{\star}(x_t, a)$ lies into, i.e., $a^\star_t \in I_{b_t}$. Our \smth regret (at level $h$) provides guarantees with respect to the smoothing kernel $\unif(I_{b_t})$. Since we have $\E_{a \sim \unif(I_{b_t})} \sq{f^\star(x_t, a)} \leq f^\star(x_t,a^\star_t)  + L h^\alpha$ under H\"older continuity, the following guarantee holds under the standard regret
\begin{equation}
    \regcb(T) \leq \regcbh(T) + L h^\alpha T. \label{eq:holder} 
\end{equation}
When $L,\alpha$ are known, setting 
$h = \Theta \prn[\big]{L^{-2/(2\alpha+1)} T^{-1/(2\alpha + 1)} \log^{1/(2\alpha+1)} \prn{\abs*{\cF}T}}$ in \cref{thm:regret} (together with \cref{eq:holder}) leads to regret guarantee $O\prn[\big]{ L^{1/(2\alpha+1)} T^{(\alpha+1)/(2\alpha+1)} \log^{(\alpha/(2\alpha+1)} \prn{\abs*{\cF}T}}$, which is nearly minimax optimal \citep{kleinberg2004nearly}. In the case when $L,\alpha$ are unknown,
setting $\eta = T^{-\beta}$ in \cref{thm:adaptive} (together with \cref{eq:holder}) leads to
\begin{align*}
    \regcb(T) 
    = O \prn*{\max \crl*{T^\beta, L^{1/(1+2\alpha)} T^{1 - \frac{\alpha}{1+\alpha}\beta} \log^{\alpha/(1+\alpha)} \prn{\abs{\cF}T}}}, 
\end{align*} 
which matches the Pareto frontier obtained in \citet{hadiji2019polynomial} up to logarithmic factors.

\section{Experiments}
\label{sec:experiments}

In this section we compare our technique empirically with prior art from the bandit and contextual bandit literature.  Code to reproduce these experiments is available at \url{https://github.com/pmineiro/smoothcb}.

\subsection{Comparison with Bandit Prior Art}

\begin{figure}
    \centering
    \includegraphics[width=0.6\linewidth]{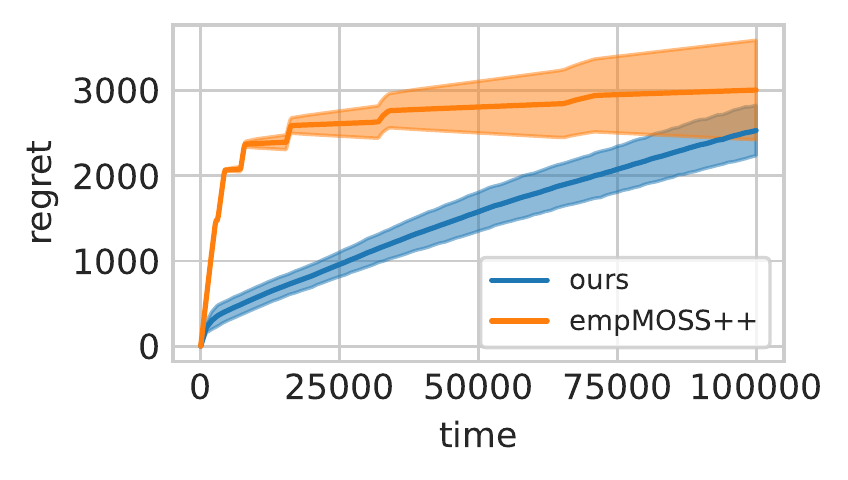}
    \caption{Comparison on a discrete action bandit dataset.  Smaller is better.  Following the display convention of \citet{zhu2020regret}, shaded areas are 38\% confidence regions.}
    \label{fig:newyorker}
\end{figure}

We replicate the real-world dataset experiment from \citet{zhu2020regret}.  The dataset consists of 10025 captions from the \emph{New Yorker Magazine} Cartoon Caption Contest and associated average ratings, normalized to [0, 1]. The caption text is discarded 
resulting in a non-contextual bandit problem with 10025 arms. When an arm is chosen, the algorithm experiences a Bernoulli loss realization whose mean is one minus the average rating for that arm.  The goal is to experience minimum regret over the planning horizon $T = 10^5$.  There are 54 arms in the dataset that have the minimal mean loss of 0.

For our algorithm, we used the uniform distribution over $[1, 2, \ldots, |\cA|]$ as a reference measure, for which $O(1)$ sampling is available.  We instantiated a tabular regression function, i.e., for each arm we maintained the empirical loss frequency observed for that arm.  
We use \corral with learning rate $\eta = 1$ and instantiated 8 subalgorithms with $\gamma h$ geometrically evenly spaced between $10^3$ and $10^6$.  These were our initial hyperparameter choices, but they worked well enough that no tuning was required.

In \cref{fig:newyorker}, we compare our technique with \texttt{empMOSS++}, the best performing technique from \citet{zhu2020regret}.  Our technique is statistically equivalent.

\subsection{Comparison with Contextual Bandit Prior Art}

We replicate the online setting from \citet{majzoubi2020efficient}, where 5 large-scale OpenML regression datasets are converted into continuous action problems on $[0, 1]$ by shifting and scaling the target values into this range.  The context $x$ is a mix of numerical and categorical variables depending upon the particular OpenML dataset.  For any example, when the algorithm plays action $a$ and the true target is $y$, the algorithm experiences loss $|y - a|$ as bandit feedback.  

We use Lebesgue measure on $[0, 1]$ as our reference measure, for which $O(1)$ sampling is available.  To maintain $O(1)$ computation, we consider regression functions with (learned) parameters $\theta$ via $
f(x, a; \theta) \ldef g\left(\wh{a}\left(x; \theta\right) - a; \theta\right)$ where, for any $\theta$, $z=0$ is a global minimizer of $g(z; \theta)$.  Subject to this constraint, we are free to choose $g(\cdot; \theta)$ and $\wh{a}(\cdot; \theta)$ and yet are ensured that we can directly compute the minimizer of our loss predictor via $\wh a(x;\theta)$. For our experiments we use a logistic loss predictor and a linear argmin predictor with logistic link: Let $\theta \ldef \prn{v; w; \xi}$, we choose $$
\begin{aligned}
g(z; \theta) \ldef \sigma\left(|w| |z| + \xi \right), \quad \text{and}
\quad \wh{a}(x; \theta) \ldef \sigma\left(v^\top x\right),
\end{aligned}
$$ where $\sigma(\cdot)$ is the sigmoid function.

\begin{table}
     \caption{Average progressive loss, scaled by 1000, on continuous action contextual bandit datasets. 95\% bootstrap confidence intervals reported.}
     \label{tab:cats}
     \vspace{0.1in}
     \centering
     \begin{tabular}{c c c c}
     \toprule
     &\texttt{\small CATS}&\texttt{\small Ours (Linear)}&\texttt{\small Ours (RFF)} \\
          \midrule
     \texttt{Cpu}&$[55, 57]$ & $[40.6, 40.7]$ & $\bold{[38.6, 38.7]}$ \\
     \texttt{Fri}&$[183, 187]$ & $[161, 163]$& $\bold{[156, 157]}$ \\
     \texttt{Price}&$[108, 110]$ & $[70.2, 70.5]$ & $\bold{[66.1, 66.3]}$ \\
     \texttt{Wis}&$[172, 174]$ & $[138, 139]$ & $\bold{[136.2, 136.6]}$ \\
     \texttt{Zur}&$[24, 26]$ & $[24.3, 24.4]$ & $[25.4, 25.5]$\\
     \bottomrule
     \end{tabular}
     
\end{table}

In \cref{tab:cats}, we compare our technique with \cats from \citet{majzoubi2020efficient}.  Following their protocol, we tune hyperparameters for each dataset to be optimal in-hindsight, and then report {95\%} bootstrap confidence intervals based upon the progressive loss of a single run.  Our algorithm outperforms \cats. 

To further exhibit the generality of our technique, we also include results for a nonlinear argmin predictor in \cref{tab:cats} (last column), which uses a Laplace kernel regressor implemented via random Fourier features~\citep{rahimi2007random} 
to predict the argmin.  
This approach achieves even better empirical performance.
 
\section{Discussion}
\label{sec:discussion}

This work presents simple and practical algorithms for contextual bandits with large---or even continuous---action spaces, continuing a line of research which assumes actions that achieve low loss are not rare. 
While our approach can be used to recover minimax/Pareto optimal guarantees under certain structural assumptions (e.g., with H\"older/Lipschitz continuity), it doesn't cover all cases.
For instance, on a large but finite action set with a linear reward function, the optimal smoothing kernel can be made to perform arbitrarily worse than the optimal action (e.g., by having one optimal action lying in an orthogonal space of all other actions); in this construction, algorithms provided in this paper would perform poorly relative to specialized linear contextual bandit algorithms.

In future work we will focus on offline evaluation.  Our technique already generates data that is suitable for subsequent offline evaluation of policies absolutely continuous with the reference measure, but only when the submeasure sample is accepted (line 4 of \cref{alg:reject_sampling}), i.e., only $M(\cA)$ fraction of the data is suitable for reuse. 
We plan to refine our sampling distribution so that the fraction of re-usable data can be increased, but presumably at the cost of additional computation.

We manage to achieve a $\sqrt{T}$-regret guarantee with respect to \smth regret, which dominates previously studied regret notions that competing against easier benchmarks.
A natural question to ask is, what is the strongest benchmark such that it is possible to still achieve a $\sqrt{T}$-type guarantee for problems with arbitrarily large action spaces?  Speculating, there might exist a regret notion which dominates \smth regret yet still admits a $\sqrt{T}$ guarantee.\looseness=-1

\bibliography{refs}

\newpage
\appendix
\section{Proofs and Supporting Results from \cref{sec:alg}}
\label{app:smooth}

This section is organized as follows. We provide supporting results in \cref{app:smooth_supporting}, then give the proof of \cref{thm:regret} in \cref{app:smooth_reg}.

\subsection{Supporting Results}
\label{app:smooth_supporting}

\subsubsection{Preliminaries}

We first introduce the concept of convex conjugate. 
For any function $\phi:\R \rightarrow \R \cup \curly*{- \infty, + \infty}$, its convex conjugate $\phi^\star: \R \rightarrow \R \cup \curly*{- \infty, + \infty}$ is defined as 
\begin{align*}
    \phi^\star(w) \ldef \sup_{v \in \R} \paren{ vw - \phi(v) }. 
\end{align*}
Since $(\phi^\star)^\star = \phi$, we have (Young-Fenchel inequality)
\begin{align}
    \phi(v) \geq vw - \phi^\star(w), \label{eq:young_fenchel}
\end{align}
for any $w \in \dom(\phi^\star)$.
\begin{lemma}
\label{prop:convex_conjugate}
$\phi(v) = \frac{1}{\gamma}(v-1)^2$ and $\phi^\star(w)=w + \frac{\gamma}{4}w^2$ are convex conjugates.
\end{lemma}
\begin{proof}[\pfref{prop:convex_conjugate}]
    By definition of the convex conjugate, we have 
    \begin{align*}
        \phi^\star(w) & = \sup_{v \in \R} \prn*{-\frac{1}{\gamma} \cdot \prn*{v^2 - \prn{2 + \gamma w }v + 1}} \\
        & = w + \frac{\gamma}{4} w^2,
    \end{align*}
    where the second line follows from plugging in the maximizer $v = \frac{\gamma w }{2} + 1$. Note that the domain of $\phi^\star(w)$ is in fact $\R^d$ here. So, \cref{eq:young_fenchel} holds for any $w \in \R^d$.
\end{proof}

We also introduce the concept of $\chi^2$ divergence. For probability measures $P$ and $Q$ on the same measurable space $(\cA, \sgm)$ such that $Q \ll P$, the $\chi^2$ divergence of $Q$ from $P$ is defined as 
\begin{align*}
    \chi^2 \infdiv{Q}{P} \ldef \E_{a \sim P} \sq*{\prn*{\frac{dQ}{dP}(a) - 1}^2},
\end{align*}
where $\frac{dQ}{dP}(a)$ denotes the Radon-Nikodym derivative of $Q$ with respect to $P$,
which is a function mapping from $a$ to $\R$.

\subsubsection{Bounding the \dectext}
We aim at bounding the \dectext in this section. 
We use expression $\inf_{Q \in \cQ_h} \E_{a^\star \sim Q} \sq{f^\star(x,a^\star)}$ for $\smthh(x)$.
With this expression, we rewrite the \dectext in the following: 
With respect to any context $x\in \cX$ and estimator $\wh f $ obtained from \sqalgtext, we denote
\begin{align*}
    \dec_\gamma(\cF; \wh f, x) \ldef \inf_{P \in \Delta(\cA)} \sup_{Q \in \cQ_h} \sup_{f \in \cF} \E_{a\sim P, a^\star \sim Q} \sq*{f(x, a) - f(x, a^\star) - \frac{\gamma}{4} \cdot \prn*{\widehat f(x,a ) - f(x, a)}^2}, 
\end{align*}
and define $\dec_\gamma(\cF) \ldef \sup_{\wh f, x } \dec_\gamma(\cF; \wh f, x)$ as the \dectext. 
We remark here that $\sup_{Q \in \cQ_h} \E_{a^\star \sim Q}   \sq{ - f(x,a^\star)} =  - \inf_{Q \in \cQ_h}\E_{a\sim Q} \sq{f^\star(x,a^\star)}$ so we are still compete with the best smoothing kernel within $\cQ_h$.

We first state a result that helps eliminate the unknown $f$ function in \dectext
(and thus the $\sup_{f\in \cF}$ term), and bound \dectext by the known $\wh f$ estimator (from the regression oracle \sqalgtext) and the $\chi^2$-divergence from $Q$ to $P$ (whenever $P$ and $Q$ are probability measures).

\begin{lemma}
    \label{lm:dec_chi2}
    Fix constant $\gamma > 0$ and context $x\in\cX$ . For any measures $P$ and $Q$ such that $Q \ll P$, we have 
    \begin{align*}
        &  \sup_{f \in \cF} \E_{a\sim P, a^\star \sim Q} \sq*{f(x, a) - f(x, a^\star) - \frac{\gamma}{4} \cdot \prn*{\widehat f(x,a ) - f(x, a)}^2} \\
        & \leq \E_{a \sim P} \sq[\big]{\widehat f(x, a)} - \E_{a \sim Q} \sq[\big]{\widehat f(x, a)} +\frac{1}{\gamma} \cdot  \E_{a \sim P}\sq*{ \prn*{\frac{dQ}{dP}(a) - 1}^2 }. 
    \end{align*}
\end{lemma}

\begin{proof}[\pfref{lm:dec_chi2}]
We omit the dependence on the context $x\in \cX$, and use abbreviations $f(a) \ldef f(x,a)$ and $\wh f(a) \ldef \wh f(x,a)$. Let $g \ldef f - \widehat f$, we re-write the expression as
\begin{align*}
    & \sup_{f \in \cF} \E_{a\sim P, a^\star \sim Q} \sq*{f( a) - f(a^\star) - \frac{\gamma}{4} \cdot \prn*{\widehat f(a ) - f(a)}^2}\\ 
    & = \sup_{g \in \cF - \widehat f} \E_{a\sim P} \sq[\big]{\widehat f(a)} - \E_{a^\star \sim Q} \sq[\big]{\widehat f(a^\star)} - \E_{a^\star \sim Q}\sq[\big]{ g(a^\star)} + \E_{a \sim P} \sq*{ g(a) - \frac{\gamma}{4} \cdot \prn*{g(a)}^2}\\
    & =  \E_{a\sim P} \sq[\big]{\widehat f(a)} - \E_{a \sim Q} \sq[\big]{\widehat f(a)} + \sup_{g \in \cF - \widehat f} \prn*{ \E_{a \sim Q}\sq[\big]{ - g(a)} -\E_{a \sim P} \sq*{ \prn{-g(a)} + \frac{\gamma}{4} \cdot \prn*{ - g(a)}^2}} \\
    & =  \E_{a\sim P} \sq[\big]{\widehat f(a)} - \E_{a \sim Q} \sq[\big]{\widehat f(a)} + \sup_{g \in \cF - \widehat f}  \E_{a \sim P}\sq*{ \frac{dQ}{dP}(a) \cdot \prn{- g(a)} -  \prn*{\prn{-g(a)} + \frac{\gamma}{4} \cdot \prn*{ - g(a)}^2}}\\
    & =  \E_{a\sim P} \sq[\big]{\widehat f(a)} - \E_{a \sim Q} \sq[\big]{\widehat f(a)} + \sup_{g \in \cF - \widehat f}  \E_{a \sim P}\sq*{ \frac{dQ}{dP}(a) \cdot \prn{- g(a)} -  \phi^\star \prn{-g(a)} },
\end{align*}
where we use the fact that $Q \ll P$ and $\phi^\star (w) = w + \frac{\gamma}{4} w^2$. Focus on the last term that depends on $g$ takes the form of the RHS of \cref{eq:young_fenchel}: Consider $v = \frac{dQ}{dP}(a)$ and $w = -g(a)$ and apply \cref{eq:young_fenchel} (with \cref{prop:convex_conjugate}) eliminates the dependence on $g$ (since it works for any $w = - g(a)$) and leads to the following bound 
\begin{align*}
    & \sup_{f \in \cF} \E_{a\sim P, a^\star \sim Q} \sq*{f( a) - f(a^\star) - \frac{\gamma}{4} \cdot \prn*{\widehat f(a ) - f(a)}^2}\\  
    & \leq \E_{a\sim P} \sq[\big]{\widehat f(a)} - \E_{a \sim Q} \sq[\big]{\widehat f(a)} + \frac{1}{\gamma} \cdot  \E_{a \sim P}\sq*{ \prn*{\frac{dQ}{dP}(a) - 1}^2 }. 
\end{align*}
\end{proof}

We now bound the \dectext with sampling distribution defined in \cref{eq:sampling_dist}. We drop the dependence on $t$ and define the sampling distribution in the generic form: Fix any constant $\gamma > 0$, context $x \in \cX$ and estimator $\wh f $, we define sampling distribution 
\begin{align}
	P \ldef M + (1 - M(\cA)) \cdot \indic_{\wh a} \label{eq:sampling_dist_gen},
\end{align}
where $\wh a \ldef \argmin_{a \in \cA} \wh f(x,a)$ and the measure $M$ is defined through $M(\omega) \ldef \int_{a \in \omega} m(a) \, d \mu(a) $ with
\begin{align}
    m(a) \ldef \frac{1}{1+ h \gamma(\widehat f(x, a) - \widehat f(x, \widehat a))}. 
    \label{eq:abe_long_measure_gen}
\end{align} 

\begin{restatable}{lemma}{propDecBound}
    \label{prop:dec_bound}
    Fix any constant $\gamma > 0$ and any set of regression function $\cF$. Let $P$ be the sampling distribution defined in \cref{eq:sampling_dist_gen}, we then have $\dec_\gamma (\cF) \leq \frac{2}{h \, \gamma}$.
\end{restatable}

\begin{proof}[\pfref{prop:dec_bound}]
    As in the proof of \cref{lm:dec_chi2}, we omit the dependence on the context $x \in \cX$ and use abbreviations $f(a) \ldef f(x,a)$ and $\wh f(a) \ldef \wh f(x,a)$.

    We first notice that for any $Q \in \cQ_h$ we have $Q \ll M$ for $M$ defined in \cref{eq:abe_long_measure_gen}: we have (i) $Q \ll \mu$ by definition, and (ii) $\mu \ll M$ (since $m(a) \geq \frac{1}{1+h\gamma} > 0$).\footnote{We thus have $Q \ll P$ as well since $P$ contains the component $M$ by definition. We will, however, mostly be working with $M$ due to its nice connection with the base measure $\mu$, as defined in \cref{eq:abe_long_measure_gen}.} 
    On the other side, however, we do not necessarily have $P \ll \mu$ for $P$ defined in \cref{eq:sampling_dist_gen}: It's possible to have $P \prn{\crl{a^\star}} > 0$ yet $\mu \prn{\crl{a^\star}} = 0$, e.g., $\mu$ is some continuous measure.
    To isolate the corner case, we first give the following decomposition for any $Q \in \cQ_h$ and $f \in \cF$.
    With $P \ldef M + (1 - M(\cA)) \cdot \indic_{\wh a}$, we have 
    \begin{align}
        & \E_{a\sim P, a^\star \sim Q} \sq*{f( a) - f(a^\star) - \frac{\gamma}{4} \cdot \prn*{\widehat f(a ) - f(a)}^2}\nonumber\\
        & = \prn{1-M(\cA)} \cdot \prn*{ f(\wh a) - \frac{\gamma}{4} \cdot \prn[\big]{\wh f(\wh a) - f(\wh a)}^2} 
        + \E_{a\sim M, a^\star \sim Q} \sq*{f( a) - f(a^\star) - \frac{\gamma}{4} \cdot \prn*{\widehat f(a ) - f(a)}^2} \nonumber\\
        & = \prn{1-M(\cA)} \cdot \prn*{ \wh f(\wh a) + \prn[\big]{f(\wh a) - \wh f(\wh a)} - \frac{\gamma}{4} \cdot \prn[\big]{\wh f(\wh a) - f(\wh a)}^2} 
        + \E_{a\sim M, a^\star \sim Q} \sq*{f( a) - f(a^\star) - \frac{\gamma}{4} \cdot \prn*{\widehat f(a ) - f(a)}^2} \nonumber\\
        & \leq \prn{1-M(\cA)} \cdot \prn[\Big]{ \wh f(\wh a) + \frac{1}{\gamma}} 
        + \E_{a\sim M, a^\star \sim Q} \sq*{f( a) - f(a^\star) - \frac{\gamma}{4} \cdot \prn*{\widehat f(a ) - f(a)}^2} \nonumber \\
        & \leq \frac{1 - M(\cA)}{\gamma} + \prn{1-M(\cA)} \cdot { \wh f(\wh a)  } 
        + \E_{a\sim M} \sq[\big]{\widehat f(a)} - \E_{a \sim Q} \sq[\big]{\widehat f(a)} + \frac{1}{\gamma} \cdot  \E_{a \sim M}\sq*{ \prn*{\frac{dQ}{dM}(a) - 1}^2 }, \label{eq:prop_dec_bound_1}
    \end{align}
    where the fourth line follows from applying AM-GM inequality and the fifth line follows from applying \cref{lm:dec_chi2} with $Q \ll M$.\footnote{With a slight abuse of notation, we use $\E_{a \sim M}[\cdot]$ denote the integration with respect to the sub-probability measure $M$.} 
We now focus on the last four terms in \cref{eq:prop_dec_bound_1}. Denote $m(a) \ldef \frac{dM}{d \mu} (a)$ and $q(a) \ldef \frac{dQ}{d \mu}(a)$, with change of measures, we have 
    \begin{align}
        & (1 - M(\cA) \cdot \prn[\big]{\wh f(\wh a)} + \E_{a\sim M} \sq[\big]{\widehat f(a)} - \E_{a \sim Q} \sq[\big]{\widehat f(a)} + \frac{1}{\gamma} \cdot  \E_{a \sim M}\sq*{ \prn*{\frac{dQ}{dM}(a) - 1}^2 } \nonumber\\
        & = \E_{a \sim \mu} \sq*{m(a) \cdot \prn[\Big]{\wh f(a) - \wh f(\wh a)} } - \E_{a \sim \mu} \sq*{q(a) \cdot \prn[\Big]{\wh f(a) - \wh f(\wh a)} } + \frac{1}{\gamma} \cdot \E_{a \sim \mu} \sq*{m(a) \cdot \prn*{\frac{q(a)}{m(a)} - 1}^2 } \nonumber \\
        & = \E_{a \sim \mu} \sq*{m(a) \cdot \prn[\Big]{\wh f(a) - \wh f(\wh a)} } - \E_{a \sim \mu} \sq*{q(a) \cdot \prn[\Big]{\wh f(a) - \wh f(\wh a)} } + \frac{1}{\gamma} \cdot \E_{a \sim \mu} \sq*{q(a) \cdot \frac{q(a)}{m(a)} - 2 q(a) + m(a)  } \nonumber \\
	& = \E_{a \sim \mu} \sq*{m(a) \cdot \prn[\Big]{\wh f(a) - \wh f(\wh a)} } + 
	\frac{1}{\gamma} \cdot \E_{a \sim Q} \brk*{\frac{q(a)}{m(a)} - \gamma \cdot \prn*{\wh f(a) - \wh f(\wh a)} } + \frac{M(\cA) - 2}{ \gamma} \nonumber \\
 \label{eq:prop_dec_bound_2}
    \end{align}
   Plugging \cref{eq:prop_dec_bound_2} into \cref{eq:prop_dec_bound_1} leads to 
   \begin{align}
	   &	   \E_{a\sim P, a^\star \sim Q}  \sq*{f( a) - f(a^\star) - \frac{\gamma}{4} \cdot \prn*{\widehat f(a ) - f(a)}^2} \nonumber \\
	& \leq
\E_{a \sim \mu} \sq*{m(a) \cdot \prn[\Big]{\wh f(a) - \wh f(\wh a)} } +
	\frac{1}{\gamma} \cdot \E_{a \sim Q} \brk*{\frac{q(a)}{m(a)} - \gamma \cdot \prn*{\wh f(a) - \wh f(\wh a)} } \nonumber \\
	& \leq  \frac{2}{h\gamma},
\label{eq:prop_dec_bound_3}
   \end{align}
   where \cref{eq:prop_dec_bound_3} follows from the fact that $m(a) \ldef \frac{dM}{d \mu}(a) = \frac{1}{1 + h \gamma \prn{\wh f(a) - \wh f(\wh a)}}$ and $q(a) \ldef \frac{dQ}{d \mu}(a) \leq \frac{1}{h}$ for any $Q \in \cQ_h$.
   This certifies that $\dec_\gamma(\cF) \leq \frac{2}{h \gamma}$.
\end{proof}

\subsection{\pfref{thm:regret}}
\label{app:smooth_reg}

\thmRegret*

\begin{proof}[\pfref{thm:regret}]
    We use abbreviation $f_t(a) \ldef f(x_t,a)$ for any $f \in \cF$.
    Let $a^\star_t$ denote the action sampled according to the best smoothing kernel within $\cQ_h$ (which could change from round to round). 
    We let $\cE$ denote the good event where the regret guarantee stated in \cref{assumption:regression_oracle} (i.e., $\regsq(T) \ldef \regsq(T, T^{-1})$) holds with probability at least  $1- T^{-1}$. Conditioned on this good event, following the analysis provided in \citet{foster2020adapting}, we decompose the contextual bandit regret as follows.
    \begin{align*}
	    \E \brk*{\sum_{t=1}^T f_t^\star(a_t) - f_t^\star(a^\star_t)}
        & = \E \sq*{\sum_{t=1}^T f_t^\star(a_t) - f_t^\star(a^\star_t) - \frac{\gamma}{4} \cdot \prn*{\wh f_t( a_t) - f_t^\star(a_t)}^2}
        + \frac{\gamma}{4} \cdot  \E \sq*{\sum_{t=1}^T \prn*{\wh f_t(a_t) - f_t^\star(a_t)}^2} \nonumber \\
        & \leq T \cdot \frac{2}{h \gamma} + \frac{\gamma}{4} \cdot  \E \sq*{\sum_{t=1}^T \prn*{\wh f_t(a_t) - f_t^\star(a_t)}^2},
    \end{align*}
    where the bound on the first term follows from \cref{prop:dec_bound}. We analyze the second term below.
    \begin{align*}
        & \frac{\gamma}{4} \cdot \E \sq*{\sum_{t=1}^T \prn*{ \prn*{\wh f_t(a_t) - \ell_t(a_t)}^2 -\prn[\Big]{f^{\star}(a_t) - \ell_t(a_t)}^2  + 2 \prn[\Big]{\ell_t(a_t) - f^\star_t(a_t)} \cdot \prn[\Big]{\wh f_t(a_t) - f^\star_t(a_t)}}} \\
        & =  \frac{\gamma}{4} \cdot \E \sq*{\sum_{t=1}^T \prn*{ \prn*{\wh f_t(a_t) - \ell_t(a_t)}^2 -\prn[\Big]{f^\star_t(a_t) - \ell_t(a_t)}^2  }} \\
        & \leq \frac{\gamma}{4} \cdot \regsq(T), 
    \end{align*}
    where on the second line follows from the fact that $\E \sq{\ell_t(a) \mid x_t} = f^\star(x_t,a)$ and $\ell_t$ is conditionally independent of $a_t$, and the third line follows from the bound on regression oracle stated in \cref{assumption:regression_oracle}. 
    As a result, we have 
    \begin{align*}
        \regcbh(T) \leq \frac{2T}{h \gamma} + \frac{\gamma}{4} \cdot \regsq(T) + O(1),
    \end{align*}
    where the additional term $O(1)$ accounts for the expected regret suffered under event  $\neg \cE$.
    Taking $\gamma = \sqrt{8 T/\prn{h \cdot \regsq(T)}}$ leads to the desired result.

    \noindent\emph{Computational complexity.}
    We now discuss the computational complexity of \cref{alg:smooth}. At each round \cref{alg:smooth} takes $O(1)$ calls to \sqalgtext to obtain estimator $\wh f_t$ and the best action $\wh a_t$. Instead of directly form the action distribution defined in \cref{eq:sampling_dist}, \cref{alg:smooth} uses \cref{alg:reject_sampling} to sample action $a_t \sim P_t$, which takes one call of the sampling oracle \samplealgtext to draw a random action and $O(1)$ calls of the regression oracle \sqalgtext to compute the mean of the Bernoulli random variable. Altogether, \cref{alg:smooth} has per-round runtime $O(\Tsq + \Tsample)$ and maximum memory $O(\Msq + \Msample)$.
\end{proof}

\section{Proofs from \cref{sec:adaptive}}
This section is organized as follows. We first prove \cref{prop:stable} in \cref{app:stable_base}, then prove \cref{thm:adaptive} in \cref{app:adaptive}.

\subsection{\pfref{prop:stable}}
\label{app:stable_base}
The proof of \cref{prop:stable} follows similar analysis as in \citet{foster2020adapting}, with minor changes to adapt to our settings.

\propStable*

\begin{proof}[\pfref{prop:stable}]
   Fix the index $b \in [B]$ of the subroutine. We use shorthands $h = h_b$, $q_t = q_{t,b}$, $\rho_t = \rho_{t,b}$, $\gamma_{t} = \gamma_{t,b}$, and so forth. 
   We also write $Z_{t} = Z_{t,b} \ldef {\ind(I_t = b)}$. Similar to the proof of \cref{thm:regret}, we use abbreviation $f_t(a) \ldef f(x_t,a)$ for any $f \in \cF$.
   Let $a^\star_t$ denote the action sampled according to the best smoothing kernel within $\cQ_h$ (which could change from round to round). 

    We let $\cE$ denote the good event where the regret guarantee stated in \cref{assumption:regression_oracle_weighted} (with $\regsq(T) \ldef \regsq(T, T^{-1})$) holds with probability at least  $1- T^{-1}$. Conditioned on this good event, similar to the proof of \cref{thm:regret} (and following \citet{foster2020adapting}), we decompose the contextual bandit regret as follows.
    \begin{align*}
    & \E \sq*{ \sum_{t = 1}^T \frac{Z_t}{q_{t}} \paren{ f^\star_t(a_t) - f_t^\star(a^\star_t) } } \\
    & = \E \sq*{ \sum_{t = 1}^T \frac{Z_t}{q_{t}} \prn*{ f^\star_t(a_t) - f_t^\star(a^\star_t)  - \frac{\gamma_t}{4}\cdot \prn*{\wh f_t(a_t) - f^\star_t(a_t)}^2} } 
    + \E \sq*{ \sum_{t=1}^T \frac{Z_t}{q_t} \cdot \frac{\gamma_t}{4}\cdot \prn*{\wh f_t(a_t) - f^\star_t(a_t)}^2}\nonumber \\
    & \leq \E \sq*{\sum_{t=1}^T \frac{Z_t}{q_t}\cdot \frac{2}{h \gamma_t}}
    + \E \sq*{ \sum_{t=1}^T \frac{Z_t}{q_t} \cdot \frac{\gamma_t}{4}\cdot \prn*{\wh f_t(a_t) - f^\star_t(a_t)}^2}\nonumber \\
    & \leq \E\sq*{\max_{t \in [T]} \gamma_t^{-1}} \cdot \frac{2T}{h} 
    + \E \sq*{ \sum_{t=1}^T \frac{Z_t}{q_t} \cdot \frac{\gamma_t}{4}\cdot \prn*{\wh f_t(a_t) - f^\star_t(a_t)}^2}\nonumber ,
    \end{align*}
    where the bound on the first term follows from \cref{prop:dec_bound} (the third line, conditioned on $Z_t$). We bound the second term next.
    \begin{align*}
    & \E \sq*{ \sum_{t=1}^T \frac{Z_t}{q_t} \cdot \frac{\gamma_t}{4}\cdot \prn*{\wh f_t(a_t) - f^\star_t(a_t)}^2}\nonumber \\
    & = \frac{1}{4} \cdot \E \sq*{\sum_{t=1}^T \frac{Z_t}{q_t} \gamma_t \prn*{ \prn*{\wh f_t(a_t) - \ell_t(a_t)}^2 -\prn[\Big]{f^\star_t(a_t) - \ell_t(a_t)}^2  + 2 \prn[\Big]{\ell_t(a_t) - f^\star_t(a_t)} \cdot \prn[\Big]{\wh f_t(a_t) - f^\star_t(a_t)}}} \\
    & =  \frac{1}{4} \cdot \E \sq*{\sum_{t=1}^T \frac{Z_t}{q_t} \gamma_t \prn*{ \prn*{f_t(a_t) - \ell_t(a_t)}^2 -\prn[\Big]{f^\star_t(a_t) - \ell_t(a_t)}^2  }} \\
    & \leq \frac{1}{4} \cdot \E \sq*{\max_{t\in[T]}\frac{\gamma_t}{q_t}} \cdot \regsq(T), 
    \end{align*}
    where the last line follows from \cref{assumption:regression_oracle_weighted}. As a result, we have 
    \begin{align*}
	    \regimph(T) \leq 
    \E\sq*{\max_{t \in [T]} \gamma_t^{-1}} \cdot \frac{2T}{h} 
    + \frac{1}{4} \cdot \E \sq*{\max_{t\in[T]}\frac{\gamma_t}{q_t}} \cdot \regsq(T) + O(1),
    \end{align*}
    where the additional $O(1)$ term is to account for the expected regret under event  $\neg \cE$.
    Notice that $\gamma_{t} \ldef \sqrt{8 T/ (h \cdot \rho_{t}
    \cdot \regsq(T))}$, which is non-increasing in $t$; and $\frac{\gamma_t}{q_t} \leq \gamma_t \rho_t$, which is non-decreasing in $t$. Thus, we have  
    \begin{align*}
	    \regimph(T) & \leq 
    \E\sq*{ \gamma_T^{-1}} \cdot \frac{2T}{h} 
    + \frac{1}{4} \cdot \E \sq*{{\gamma_T}{\rho_T}} \cdot \regsq(T) + O(1)\\ 
    & = \E \sq*{\sqrt{\rho_T}} \cdot \sqrt{T \regsq(T)/ 2 h} + \E \sq*{\sqrt{\rho_T}} \sqrt{T \regsq(T)/2h} + O(1)\\
    & \leq \E \sq*{\sqrt{\rho_T}} \cdot  \sqrt{4T \regsq(T)/ h}.
    \end{align*}

    \emph{Computational complexity.} The computational complexity of \cref{alg:stable} can be analyzed in a similar way as the computational complexity of \cref{alg:smooth}, except with a \emph{weighted} regression oracle \sqalgtext this time.
\end{proof}

\subsection{\pfref{thm:adaptive}}
\label{app:adaptive}

We first restate the guarantee of \corral, specialized to our setting.

\begin{theorem}[\citet{agarwal2017corralling}]
    \label{thm:corral}
    Fix an index $b \in [B]$. Suppose base algorithm $b$ is $(\alpha_b, R_b(T))$-stable with respect to decision space indexed by $b$. If $\alpha_b < 1$, the \corral master algorithm, with learning rate $\eta>0$, guarantees that 
    \begin{align*}
        \E \sq*{\sum_{t=1}^T f^\star(x_t,a_t) - \inf_{Q_t \in \cQ_{h_b}} \E_{a^\star_t \sim Q_t} \sq*{f^\star(x_t, a^\star_t)} } = \wt O \prn*{\frac{B}{\eta} + T \eta + \prn*{R_b(T)}^{\frac{1}{1 - \alpha_b}} \eta^{\frac{\alpha_b}{1-\alpha_b}}}.
    \end{align*}
\end{theorem}

\thmAdaptive*
\begin{proof}[\pfref{thm:adaptive}]
    We prove the guarantee for any $h^\star \in [1/T,1]$ as the otherwise the bound simply becomes vacuous.
    Recall that we initialize $B = \ceil{\log T}$ \cref{alg:stable} as base algorithms, each with a fixed smoothness parameter $h_b = 2^{-b}$, for $b \in [B]$. 
    Using such geometric grid guarantees that there exists an $b^\star \in [B]$ such that $h_{b^\star} \leq h^\star \leq 2 h_{b^\star}$. 
    To obtain guarantee with respect to $h^\star$, it suffices to compete with subroutine $b^\star$ since $\cQ_{h^\star} \subseteq \cQ_{h_{b^\star}}$ by definition. \cref{prop:stable} shows that the base algorithm indexed by $b^\star$ is $(\frac{1}{2}, \sqrt{4 T \regsq(T)/ h_{b^\star}})$-stable. Plugging this result into \cref{thm:corral} leads to the following guarantee:
    \begin{align*}
        \E \sq*{\sum_{t=1}^T f^\star(x_t,a_t) - \inf_{Q_t \in \cQ_{h^\star}} \E_{a^\star_t \sim Q_t} \sq*{f^\star(x_t, a^\star_t)} } & 
        \leq \E \sq*{\sum_{t=1}^T f^\star(x_t,a_t) - \inf_{Q_t \in \cQ_{h_{b^\star}}} \E_{a^\star_t \sim Q_t} \sq*{f^\star(x_t, a^\star_t)} } \\ 
        & = \wt O \prn*{\frac{B}{\eta} + T \eta + \frac{\eta \, T \, \regsq(T)}{h_{b^\star}}} \\
        & = \wt O \prn*{\frac{1}{\eta} + T \eta + \frac{\eta \, T \, \regsq(T)}{h^\star}} .
    \end{align*}

    \emph{Computational complexity.}
    The computational complexities (both runtime and memory) of the \corral master algorithm can be upper bounded by $\wt O(B \cdot \cC)$ where we use $\cC$ denote the complexities of the base algorithms. We have $B = O(\log T)$ in our setting. Thus, directly plugging in the computational complexities of \cref{alg:stable} leads to the results. 
\end{proof}

\subsubsection{Recovering Adaptive Bounds in \citet{krishnamurthy2020contextual}}
\label{app:adaptive_2}
We discuss how our algorithms can also recover the adaptive regret bounds stated in \citet{krishnamurthy2020contextual} (Theorems 4 and 15), i.e.,
\begin{align*}
    \regcbh(T) = \wt O \prn*{ T^{\frac{1}{1+\beta}} (h^\star)^{-\beta} \prn*{\log \abs*{\cF }}^{\frac{\beta}{1+ \beta}}},
\end{align*}
for any $h^\star \in (0,1]$ and $\beta \in [0, 1]$. This line of analysis directly follows the proof used in \citet{krishnamurthy2020contextual}. 

We focus on the case with $\regsq(T) = O \prn{\log \prn{\abs{\cF}T}}$.
For base algorithm (\cref{alg:stable}), following the analysis used in \citet{krishnamurthy2020contextual}, we have 
    \begin{align*}
	    \regimph(T) &
	    \leq   \min \crl*{T, \E \sq*{\sqrt{\rho_T}} \cdot \sqrt{4 T \regsq(T)/ h}} \\
     & \leq \min \crl*{T,  \sqrt{\E \sq{\rho_T}}\cdot \sqrt{4 T \regsq(T)/ h}} \\
     & = O \prn*{  T^{\frac{1}{1 + \beta}} \cdot \prn*{\E \sq{\rho_T} \regsq(T) / h}^{\frac{\beta}{1 + \beta}} },
    \end{align*}
    where on the first line we combine the regret obtained from \cref{prop:stable} with a trivial upper bound $T$; on the second line we use the fact that $\sqrt{\cdot}$ is concave; and on the third line we use that fact that $\min \crl{A,B} \leq A^\gamma B^{1 - \gamma}$ for $A,B > 0$ and $\gamma \in [0,1]$ (taking $A= T$, $B = \sqrt{\E \sq{\rho_T} \cdot 4 T \regsq(T) / h}$ and $\gamma = \frac{1 - \beta}{1 + \beta}$). This line of analysis thus shows that \cref{alg:stable} is $\prn*{\frac{\beta}{1+\beta}, \wt O \prn*{  T^{\frac{1}{1 + \beta}} \cdot \prn*{ \regsq(T) / h}^{\frac{\beta}{1 + \beta}} }}$-stable for any $\beta \in [0,1]$.\footnote{As remarked in \citet{krishnamurthy2020contextual}, the \corral algorithm works with both $\E \sq{\rho_T^\alpha}$ and $\prn*{ \E \sq{\rho_T}}^\alpha$.}
    
    Now following the similar analysis as in the proof of \cref{thm:adaptive}, and consider $\regsq(T) = O(\log \prn{\abs{\cF}T})$ for the case with a finite set of regression functions, we have 
        \begin{align*}
        \E \sq*{\sum_{t=1}^T f^\star(x_t,a_t) - \inf_{Q_t \in \cQ_{h^\star}} \E_{a^\star_t \sim Q_t} \sq*{f^\star(x_t, a^\star_t)} } 
        = \wt O \prn*{\frac{1}{\eta} + T \eta + T \cdot \prn*{\frac{\log \prn{\abs*{\cF}T} \, \eta}{h^\star}}^\beta},
    \end{align*}
    for any $h^\star \in (0,1]$. Taking $\eta = T^{- \frac{1}{1+\beta}}\cdot \prn*{\log \prn{ \abs*{\cF}T}}^{- \frac{\beta}{1+\beta}}$ recovers the results presented in \citet{krishnamurthy2020contextual}.

\end{document}